%% file: main.tex
\documentclass{article}

\input{maths_commands.tex}

\usepackage{microtype}
\usepackage{graphicx}
\usepackage{subfigure}
\usepackage{booktabs} % for professional tables
\usepackage{shuffle}

\usepackage{hyperref}

\usepackage[accepted]{icml2023}

\usepackage{amsmath}

\usepackage{amssymb}
\usepackage{mathtools}
\usepackage{amsthm}

\usepackage[capitalize,noabbrev]{cleveref}

\theoremstyle{plain}
\newtheorem{thm}{Theorem}[section]

\theoremstyle{definition}

\theoremstyle{remark}
\newtheorem{rem}[thm]{Remark}

\def\bfX{\boldsymbol{X}}
\newcommand{\bbR}{\mathbb R}
\newcommand{\bbP}{\mathbb P}

\usepackage[textsize=tiny]{todonotes}

\icmltitlerunning{Submission and Formatting Instructions for ICML 2023}

\begin{document}

\twocolumn[
\icmltitle{A Neural RDE approach for continuous-time  non-Markovian \\ stochastic control problems}




\begin{icmlauthorlist}
\icmlauthor{Melker Höglund}{imperial}
\icmlauthor{Emilio Ferrucci}{oxford}
\icmlauthor{Camilo Hernández}{princeton}
\icmlauthor{Aitor Muguruza Gonzalez}{imperial,kaiju}
\icmlauthor{Cristopher Salvi}{X,turing}
\icmlauthor{Leandro Sánchez-Betancourt}{kings}
\icmlauthor{Yufei Zhang}{lse}
\end{icmlauthorlist}

\icmlaffiliation{X}{Department of Mathematics, I-X, Imperial College London.}
\icmlaffiliation{imperial}{Department of Mathematics, Imperial College London.}
\icmlaffiliation{oxford}{Mathematical Institute, University of Oxford.}
\icmlaffiliation{princeton}{ORFE Department, Princeton University.}
\icmlaffiliation{turing}{The Alan Turing Institute.}
\icmlaffiliation{kaiju}{Kaiju Capital Management.}
\icmlaffiliation{kings}{Department of Mathematics, King's College London}
\icmlaffiliation{lse}{Department of Statistics, London School of Economics.}

\icmlcorrespondingauthor{Cristopher Salvi}{c.salvi@imperial.ac.uk}

\icmlkeywords{Machine Learning, ICML}


\vskip 0.3in
]



\printAffiliationsAndNotice{\icmlEqualContribution} % otherwise use the standard text.

\begin{abstract}
We propose a novel framework for solving continuous-time non-Markovian stochastic optimal  problems by means of \emph{neural rough differential equations} (Neural RDEs) introduced in \citet{morrill2021neural}.
Non-Markovianity naturally arises in control problems due to the time delay effects in the system coefficients or the driving noises, which   leads to   optimal control strategies  depending explicitly on the historical trajectories of the system state.
By modelling the control process as the solution of a Neural RDE driven by the state process, we show that the control-state joint dynamics are governed by an \emph{uncontrolled, augmented} Neural RDE, allowing for fast Monte-Carlo estimation of the value function via trajectories simulation and memory-efficient back-propagation. 
We provide theoretical underpinnings for the proposed algorithmic framework by   demonstrating  that Neural RDEs serve as universal approximators for functions of random rough paths.
% To deal with input paths of infinite 1-variation, we refine the universal approximation in \citet{kidger2020neural} to a probabilistic density result for Neural RDEs driven by random rough paths. 
Exhaustive numerical experiments on  non-Markovian stochastic control problems are presented, which reveal that   the proposed framework is time-resolution-invariant and 
 achieves higher accuracy and better stability in irregular sampling compared to existing RNN-based approaches.
% learns optimal controls with higher accuracy than existing  RNN-based approaches. 
%Finally, we discuss possible extensions of this framework to the setting of non-Markovian, continuous-time reinforcement learning and provide promising empirical evidence in this direction.
\end{abstract}

\section{Introduction}

The field of stochastic control is concerned with problems where an agent interacts over time with some random environment through the action of a \textit{control}. In this setting, the agent seeks to select the control such that some objective depending on the trajectory of the system under their control and the choice of the control itself is optimised; commonly, as the system is stochastic, such an objective takes the form of an expectation of some pathwise cost or reward. The study of this class of problems % is intimately related to reinforcement learning (RL) and 
has been successfully applied to many fields of modern sciences, including biology \citet{cucker2007emergent}, economics \citet{kamien2012dynamic}, engineering \citet{grundel2007cooperative}, finance \citet{pham2009continuous}, and more recently, epidemics control \citet{hubert2022incentives}.

Stochastic control is nowadays regarded as a well-established field of mathematics. Two main approaches govern the analysis: the stochastic maximum principle and the dynamic programming approach, see \citet{yong1999stochastic,pham2009continuous}. In either case, an agent is interested in characterising a set of optimal strategies, the dynamics of the system under such strategies, and the optimal value of the corresponding reward  functional. The two main sources of complexity for tackling these problems are: 1) the continuous-time nature of the underlying stochastic dynamics, and 2) the presence of memory yielding a non-negligible impact of the system's history on its future evolution.

% On the one hand, compared to their discrete counterparts, 
% Continuous-time stochastic control problems have received an increasing amount of attention in recent years, partly because the underlying physical processes themselves often develop in continuous time.  
% partly because of their characterisation via partial differential equations (PDEs) or backward stochastic differential equations (BSDEs).
%On the other hand, 
Continuous-time non-Markovian stochastic control problems, where the evolution of the system depends on its history and not only on its current state, 
 have received an increasing amount of attention in recent years. 
 Non-Markovian system 
  provides a more faithful class of models to describe real-world phenomena than their Markov counterparts, where the (infinitesimal) displacement of the state dynamics depend  only on the current state.
Non-Markovianity naturally arises in control problems due to the time delay effects in the system coefficients or the driving noises, which   leads to the optimal control strategy being influenced by the historical trajectories of the system states.

Typical examples of continuous-time  non-Markovian stochastic control problems  include rough volatility models \citet{gatheral2018volatility} from quantitative finance in which the non-Markovianity stems from having a fractional Brownian motion as the driving noise. 
Fractional Brownian motion generalises   Brownian motion and  involves history-dependent increments. As a result, the state dynamics   driven by fractional Brownian motion exhibit non-Markovian behavior.
Another example   of non-Markovian problems are delayed control
problems, where memory is incorporated into the system by assuming path-dependence of the vector fields governing the dynamics (see Sec. \ref{sec:method} for a precise statement). Optimal decision with time delay  is ubiquitous in economics, for example in the study of growth models with delayed production or pension funds models, \citet{kydland1982time, salvatore2011stochastic}, in marketing for models of optimal advertising with distributed lag  effects \citet{gozzi2009controlled}, and in finance for portfolio selection under the market with memory and delayed responses \citet{oskendal2011optimal}. See also \citet{kolmanovski1996control} for modelling systems with after-effect in mechanics, engineering, biology, and medicine.

As the solution to a continuous-time non-Markovian stochastic control problem   is in general not known analytically, it is important to construct effective and robust numerical schemes for solving these control problems.
An essential numerical challenge is to effectively capture the nonlinear dependence of the optimal control strategy on the historical trajectories of the system states.

% as well as the fact that some problems admit a close form analytic solution, for instance the so-called Merton problem in finance \citet{Merton71} or the linear quadratic regulator in engineering \citet{yong1999stochastic}. The latter property is frequently lost  in discrete frameworks with either a large number of periods or possible states/actions.

%Despite recent theoretical advances in simplified settings, non-Markovian stochastic control problems in continuous-time are often not analytically tractable, a fact that undeniably motivates the need for developing efficient numerical schemes to solve them. Additionally, such methods could provide a fruitful basis for (non-Markovian) reinforcement learning in continuous time, studied in the Markovian case recently by \citet{jia2021policy,wang2020reinforcement}. 

\paragraph{Contributions} Using the modern tool set offered by neural rough differential equations (Neural RDEs) \citet{morrill2021neural} ---  a continuous-time analogue to recurrent neural networks (RNNs) --- we propose a novel framework which, to the best of our knowledge, is the first numerical approach allowing to solve non-Markovian stochastic control problems in continuous-time. More precisely, we parameterise the control process as the solution of a Neural RDE driven by the state process, and show that the control-state joint dynamics are governed by an uncontrolled RDE with  vector fields parameterised by neural networks. We demonstrate how this formulation allows for trajectories sampling, Monte-Carlo estimation of the reward functional and backpropagation. To deal with sample paths of infinite 1-variation, which is necessary in stochastic control, we also extend the universal approximation result in \citet{kidger2020neural} to a probabilistic density result for Neural RDEs driven by random rough paths. The interpretation is that we are able to approximate continuous feed-back controls arbitrarily well in probability. Through various experiments, we demonstrate that  the proposed framework is time-resolution-invariant and capable of learning optimal controls with higher accuracy than traditional RNN-based approaches. 
%Finally, we discuss possible extensions to the setting of non-Markovian reinforcement learning (RL) in continuous-time and provide promising empirical evidence in this direction.

The rest of the paper is organised as follows: in Sec. \ref{sec:related_work} we discuss some related work, in Sec. \ref{sec:method} we present our algorithmic framework  and the universal approximation  result of Neural RDEs, 
%in Sec. \ref{sec:RL} we study the extension to non-Markovian RL in continuous-time, 
and in Sec. \ref{sec:experiments} we demonstrate the effectiveness of the algorithm through numerical experiments.

\section{Related work}\label{sec:related_work}

Over the last decade, a large volume of research has been conducted to solve Markovian stochastic control problems numerically using neural networks, either by directly parameterising the control and then sampling from the state process, such as done by \citet{han2016deep}, or by solving the  
partial differential equations 
(PDEs) or
backward stochastic differential equations
(BSDEs) associated with the problem; see \citet{germain2021neural} for a recent survey about neural networks-based algorithms for stochastic control and PDEs. We also mention two examples from the growing literature. The Deep BSDE model from \citet{han2017deep}, where the authors propose an algorithm to solve parabolic PDEs and BSDEs in high dimension and  think of the gradient of the solution as the policy function,  approximated with a neural network. The Deep Galerkin model \citet{sirignano2018dgm} is a mesh-free algorithm to solve PDEs associated with the value function of control problems; the authors approximate the solution with a deep neural network which is trained to satisfy the PDE differential operator, initial condition, and boundary conditions.

Recently, signatures methods \citet{lyons2014rough,kidger2019deep} have been employed for solving both Markovian and non-Markovian control problems in simplified settings \citet{kalsi2020optimal, cartea2022optimal}. This approach does not rely on a model underpinning the dynamics of the unaffected processes and has shown excellent results when solving a number of algorithmic trading problems.  However, this method has two main drawbacks: (i) it suffers from the curse of dimensionality — this happens when one wishes to compute signatures of a high-dimensional (more than five) process to make online decisions, and (ii) it requires that the flow of information observed by the controller is unaffected by the control and everything else the controller observes can be explicitly constructed from such information and the policy. We also point out the theoretical contribution by \citet{DFG17} studying control problems where the driving noise is a random rough path.

The approach of directly parameterising the control and training by sampling trajectories from the system was recently studied in the setting of delay-type non-Markovian stochastic control by \citet{han2021recurrent}. Specifically, the control is taken to be a Long Short-Term Memory (LSTM) recurrent neural network with the discrete simulated values of the state process as input, so as to capture the path-dependence of the problem. The method is shown to outperform a baseline parameterisation using a fully-connected feed-forward network taking as input a segment of the history of the sample path, and demonstrated to have theoretical advantages in handling non-Markovian problems.

Neural RDEs, as popularised by  \citet{kidger2020neural, morrill2021neural} provide an elegant way of modelling temporal dynamics by parameterising the vector fields of some classes of differential equations by neural networks. The input to such models is a multivariate time series interpolated into a continuous path $X$. Depending on the level of (ir)regularity of $X$, the corresponding equation can be solved in different ways. In \citep{kidger2020neural}, $X$ is assumed differentiable almost everywhere, and the equation becomes an ordinary differential equation (ODE) that can be evaluated numerically via a call to an ODE solver of choice. More generally, if $X$ is of bounded variation, then the Neural RDE can be solved using classical Riemann–Stieltjes or Young integration \citep{young1905vi}. 

Of particular interest in the field of stochastic control is the setting where the driving noise is Brownian motion , and the resulting dynamical systems are typically referred to as \emph{stochastic differential equations} (SDEs). Because sample paths from  Brownian motion are not of bounded variation, the integral cannot be interpreted in the classical sense, but rather using the framework of stochastic integration (It\^o, Stratonovich, etc.). The corresponding "neural" version of such models has been the object of several studies \citet{liu2019neural, li2020scalable, kidger2021neural, kidger2021efficient}, in particular in the context of generative modelling for time series. Rough integration \citet{lyons1998differential} is arguably the most general type of integration theory accommodating driving signals $X$ of arbitrary roughness, and in particular non-Markovian processes such as fractional Brownian motion. In this paper, we position ourselves in this general setting. In the appendix, we provide a minimal summary of the basic notions of this theory underpinning the content of this paper.

\section{Problem formulation and methodology}\label{sec:method}

\subsection{Control problems with path-dependent coefficients}

Let us introduce the non-Markovian control problems over closed-loop controls. We fix %nonnegative integers $d$, $d_a$, $d_W$,
$d, d_a, d_W \in \mathbb{N}$,
a real number $T>0$ and $\mathcal{C}^d:=\mathcal{C}([0,T];\sR^d)$, the space of continuous paths from $[0,T]$ to $\mathbb{R}^d$ endowed with the $\sup$ norm. 
Let $(\Omega, \mathcal{F}, \mathbb{P})$ be a probability space supporting
a $d_W$-dimensional Brownian motion $W=(W_t)_{t\in [0,T]}$, 
and $\mathbb{F}$ be the natural filtration of $W$ augmented with the $\mathbb{P}$-null sets.
Let $\mathcal{H}^2(\sR^{d_a})$ be the space of all  square integrable $\mathbb{F}$-progressively measurable processes, and for each $\alpha \in \mathcal{H}^2(\sR^{d_a})$,
consider the  following controlled state dynamics: for all $t\in [0,T]$,
\begin{equation}\label{eqn:controlled_RDE}
    \D X_t = \mu(t, {X}_{\cdot\wedge t},\alpha_t) \D t + \sigma (t, {X}_{\cdot\wedge t}, \alpha_t)\D W_t,
\end{equation}
where $X_0=x_0$, ${X}_{\cdot\wedge t} = \{X_s\}_{s\in [0,t]}$, and 
$(\mu,\sigma):[0,T]\times \mathcal{C}^d\times \sR^{d_a}\longrightarrow \sR^{d}\times \sR^{d\times d_W}$ are non-anticipative and sufficiently regular mappings 
so that 
\eqref{eqn:controlled_RDE} admits a unique solution 
$X$ in $\mathcal{H}^2(\sR^{d})$.\footnote{This is the case if, for instance, $\mathcal{C}^d\times A\ni(x,a) \longmapsto \varphi(t,x,a)$ has linear growth and is Lipschitz continuous uniformly in $t$ for $\varphi=\mu,\sigma$, see \citet{protter2005stochastic}.}  We denote by $\mathcal{A}$ the  set of admissible controls 
containing all $\alpha\in \mathcal{H}^2(\sR^{d_a})$ that are adapted to the filtration generated by $X$.
Such controls are often referred to as closed-loop, or feedback, controls. 

The agent's goal is to minimise the following cost functional
\begin{equation}\label{eqn:functional}
    J(x_0,\alpha) = \EE\left[\int_0^T f(t,{X}_{\cdot\wedge t},\alpha_t)\D t + g({X}_{\cdot\wedge T})\right]
\end{equation}
over all  
%closed-loop 
controls  $\alpha \in \mathcal{A}$, where 
$f:[0,T]\times \mathcal{C}^d\times \sR^{d_a}\to \sR$
and $g: \mathcal{C}^d \to \sR$ are given measurable functions.
Note that  
\eqref{eqn:functional}
is a non-Markovian control problem, 
as  the coefficients of the state dynamics and the cost functions depend   on the history of the system state.
Hence, the optimal control process  will also depend on the   entire state trajectory, 
instead of the current system state. 
\subsection{Policy parametrisation via  Neural RDE} 
\label{sec:model-based}
Here, we are going to parameterise the control process $\alpha$ in equation (\ref{eqn:controlled_RDE}) as the solution of a Neural RDE driven by the state process $X$. Let $\ell_\theta:\mathbb{R}^{d_a} \to \mathbb{R}^{d_h}, h_\theta : \mathbb{R}^{d_h} \to \mathbb{R}^{d_h \times d}, A_\theta \in \mathbb{R}^{d \times d_h}$ be (Lipschitz) neural networks. Collectively, they are parameterised by $\theta$. The dimension $d_h>0$ is a
hyperparameter describing the size of the hidden state.

We parameterise controls $\alpha^\theta \in \mathcal{A}$ as solutions to Neural RDEs driven by $X$,
\begin{equation}\label{eqn:neural_CDE}
    Y_0 = \ell_\theta(x_0), \quad
    \D Y_t = h_\theta(Y_t)\D X_t, \quad
    \alpha_t^\theta = A_\theta Y_t.
\end{equation}
With this choice of parameterisation, the dynamics of the joint process $(X,Y)$ are governed by the following \textit{uncontrolled} RDE with structured vector fields
\begin{align} \label{eqn:augmented_RDE}
    \D\begin{pmatrix}
    X \\ Y 
    \end{pmatrix}_t &= 
    \mu\left(t, {X}_{\cdot\wedge t}, A_\theta Y_t\right)\begin{pmatrix}
    1 \\
    h_\theta(Y_t)
    \end{pmatrix} \D t \\
    &\quad + 
    \sigma\left(t, {X}_{\cdot\wedge t}, A_\theta Y_t\right)\begin{pmatrix}
    I_{d} & 0 \\
    0 & h_\theta(Y_t)
    \end{pmatrix} \D \begin{pmatrix} W \\ W  \end{pmatrix}_t \nonumber
\end{align}
Thus, the infinite dimensional minimisation over admissible controls of the reward functional $J$ in equation (\ref{eqn:functional}) can be replaced with the finite-dimensional minimisation over the parameters $\theta$ of the following objective functional
\begin{equation}\label{eqn:goal_functional}
    J(x_0,\alpha^\theta) = \EE\left[\int_0^T f(t,{X}_{\cdot\wedge t},A_\theta Y_t)\D t + g({X}_{\cdot\wedge T})\right],
\end{equation}
Here, we perform this minimisation by first solving numerically the uncontrolled Neural RDE (\ref{eqn:augmented_RDE}) using a classical Euler-Maruyama scheme\footnote{For convergence guarantees of Euler-Maruyama schemes applied to SDEs with path-dependent vector fields we refer the reader to \citet{mao2003numerical}. Other choices of solvers are possible.}; we then use the obtained sample trajectories to compute a Monte-Carlo estimate of the objective functional in (\ref{eqn:goal_functional}), where the integral is approximated using classical quadrature; finally we compute gradients of the estimated objective functional with respect to model parameters $\theta$ and optimise by (stochastic) gradient descent.

Contrary to the approach taken by \citet{han2021recurrent} using an LSTM-parameterisation of the control, our formulation does not rely on any specific discretisation or choice of numerical method. A key feature of Neural RDEs is their robustness to irregular sampling of the data, essentially due to operating continuously in time. The sampled data enters the model only through the construction of the interpolated path, after which the RDE can be solved numerically on any desired grid using adaptive schemes that changes the step size to appropriately resolve the variations in the path. Therefore, because our scheme can be formulated completely in continuous-time and independently of whichever way one chooses to estimate $J(\alpha)$, it is naturally \emph{time-resolution invariant}, in the sense  that even if trained on a coarser resolution it can be directly evaluated on a finer resolution without retraining.

\subsection{Extension to problems with non-Markovian  noises} 

The algorithm proposed in Section \ref{sec:model-based}
can be easily applied  to other types of  non-Markovian control problems. 
For instance, consider minimising the following cost functional:
\begin{equation} 
\label{eqn:loss_fractional}
    J(x_0,\alpha) = \EE\left[\int_0^T f(t,{X}_{  t},\alpha_t)\D t + g({X}_{  T})\right],
\end{equation}
subject to the 
  following state dynamics (cf.~\eqref{eqn:controlled_RDE}): 
for all $t\in [0,T]$,
\begin{equation}\label{eqn:controlled_RDE_fractional}
    \D X_t = \mu(t, {X}_{ t},\alpha_t) \D t + \sigma (t, {X}_{ t}, \alpha_t)\D W^H_t,
\end{equation}
where $(\mu,\sigma):[0,T]\times \sR^d\times \sR^{d_a}\longrightarrow \sR^{d}\times \sR^{d\times d_W}$ are given   functions, 
and $W^H$ is a $d_W$-dimensional fractional Brownian motion with  Hurst index  $H\in (0,1)$ defined by
$$
W^H_t = \frac{1}{\Gamma(H+\frac{1}{2})}\int_0^t (t-s)^{H-\frac{1}{2}}dW_s, 
$$
with $\Gamma(\cdot) $ being the Gamma function.

Note that the distribution of the noise increment $W^H_{t+\delta}-W^H_t$, $\delta>0$, depends on the   trajectory   $(W^H_s)_{s\le t}$, 
and hence the state process $X$ is  non-Markovian, even if all coefficients of 
  \eqref{eqn:loss_fractional}
  and \eqref{eqn:controlled_RDE_fractional}
depend only on the current state and control variables. 
As a result, the optimal feedback control of  
\eqref{eqn:controlled_RDE_fractional}
typically depend on the entire history of the state process (see e.g., \cite{duncan2010stochastic}). 

% Up until now we have considered stochastic control problems where the non-Markovianity stems from some path-dependence of the coefficients on the history of the system. Indeed, the heuristic meaning of \eqref{eqn:controlled_RDE} is that the infinitesimal increment $X_{t+\D t} - X_t$ is normally distributed with mean $\mu(t, {X}_{\cdot\wedge t},\alpha_t)\D t$ and variance $\sigma(t,\ {X}_{\cdot\wedge t},\alpha_t)^2 \D t$, and independent of the past $\mathcal F_t$: the solution is not Markovian because knowing the value of $X_t$ does not contain all the information necessary to evaluate the path-dependent coefficients $\mu$ and $\sigma$, which are needed to compute mean and variance of the increment. There is a second, very different way in which the process $X$ can fail to satisfy the Markov property: consider \eqref{eqn:controlled_RDE} in the case in which $W$ is a process with correlated increments such as fractional Brownian motion $W^H$. Now Markovianity does not hold because the increment $X_{t+\D t} - X_t$ is not independent of the past, even if $\mu$ and $\sigma$ are state-dependent: this is because the noise increment $\D W^H$ is correlated with its own history, even conditionally on the present. 
However, 
the approach in Sec. \ref{sec:model-based}
can be naturally extended to this setting,
as the algorithm directly parameterises the feedback controls
and hence is invariant with respect to different driving noises.
In contrast, as in the case of path-dependent coefficients, the classical methods for Markovian control problems will result in  sub-optimal policies in the  
presence of fractional noises. 
% approaches valid in the state-dependent case do not work in the fractional case without major modification.

\subsection{Universality of Neural RDEs}\label{sec:universality}

The section  provides theoretical underpinnings for the proposed algorithmic framework by   demonstrating  that Neural RDEs serve as universal approximators for functions of random rough paths.

More precisely, we   prove the density  (in  probability) of linear functionals on the signature of rough paths, which implies the universality of   Neural RDEs (see  \autoref{rmk:universality_RDE}). We refer the reader to \autoref{sub:backRps} for basic backgrounds of rough path theory, which will be used in this section. Let $\alpha \in (0,1]$ and $\bfX \colon \Omega \times [0,T] \to T^{\lfloor 1/\alpha \rfloor}(\bbR^{1+d})$ be a stochastic $\alpha$-H\"older rough path with the property that the zero-th component of its trace is the time coordinate, $X^0_t = t$, and whose higher components that involve the zero-th are defined canonically through Stieltjes integration.
In practice,   $\bfX$ can  represent either the state trajectory, or  the underlying driving (fractional) Brownian noise. 

%in 
% \eqref{eqn:controlled_RDE}
% and 
% \eqref{eqn:controlled_RDE_fractional},
% augmented with the time variable. 

The following \autoref{thm:univ} is the probabilistic analogue of a well-known universal approximation  property of the deterministic signature, see for instance Proposition 3 in \citet{Ferm21}.\footnotemark
\footnotetext{It is not to be confused with the universality property of the expected signature with respect to functions of distributions on paths, see Theorem 3.2 in \citet{LSDBL21}.}
The main contribution  of \autoref{thm:univ}, compared to its deterministic counterpart (e.g.\ see \citet{kidger2020neural}), is that it avoids imposing a compactness assumption and establishes the universality  in probability.  
As a result, the application of  \autoref{thm:univ} is simplified, as it eliminates the need to select a compact set in the space of $\alpha$-H\"older continuous paths, which   typically entails additional boundedness and smoothness conditions of the paths.

% on the other hand, it does require a probability measure on path-space supported in the space of $\alpha$-H\"older continuous paths, and universality holds in probability instead of in a pathwise sense.

\begin{thm}\label{thm:univ}
	Let $\beta < \alpha$, $\bfX$ as above, and $F \colon  C^\beta([0,T], \bbR^{1+d}) \to \bbR$ be a continuous map. Then for each $\varepsilon, \delta > 0$ there exists a truncation level $N$ and a linear map $\ell \in T^N(\bbR^{1+d})$ such that
	\begin{equation}\label{eq:denseSig}
		\bbP \big[ |F(\bfX) - \langle \ell, S^N(\bfX)_{0T} \rangle| \geq \varepsilon \big] < \delta,
	\end{equation}
 where $S^N(\bfX)$ refers to the truncated signature of $\bfX$  (cf.~\autoref{sub:backRps}).
\end{thm}
\begin{proof}
		Let $D^\alpha_r$ be the closed disk centred at the $0$ rough path, of radius $r>0$ in $C^\alpha([0,T],\bbR^{1+d})$. Since $\bfX$ is a.s.\ $\alpha$-H\"older continuous in the rough path sense 
		\begin{align*}
		\lim_{r \to \infty} \bbP[\bfX \in D^\alpha_r] &= \bbP\Big[ \bigcup_{r > 0} \bfX^{-1}( D^\alpha_r) \Big] \\&= \bbP\Big[ \bfX^{-1}(\bigcup_{r > 0} D^\alpha_r) \Big] \\&= \mathbb P[\bfX \in C^\alpha([0,T],\bbR^{1+d})] \\&= 1
		\end{align*}
		and thus given $\delta$ as in the statement there exists $r$ s.t.\ $\bbP[\bfX \in D^\alpha_r] > 1-\delta$. By Proposition 8.17 (ii) in \citet{FV10}, for any $\beta < \alpha$, $D^\alpha_r$ is sequentially compact in $ C^\beta([0,T],\bbR^{1+d})$, and thus compact since this is a metric space. Let $\widetilde D^\alpha_r$ be the intersection of $D^\alpha_r$ with the aforementioned set of rough paths in $C^\alpha([0,T],\bbR^{1+d})$ whose zero-th coordinate is time $t$; this is a closed set and thus $\widetilde D^\alpha_r$ is still compact. Thanks to the inclusion of the time coordinate, linear functions on the signature separate points in $\widetilde D^\alpha_r$, and by the Stone-Weierstrass theorem applied to $F|_{\widetilde D^\alpha_r}$ there exist $N$ and $\ell$ as in the statement s.t.\ $|F(\bfX(\omega)) - \langle \ell, S^N(\bfX(\omega))_{0T} \rangle| < \varepsilon$ for $\omega \in \Omega$ s.t.\ $\bfX(\omega) \in D^\alpha_r$, and the conclusion now follows.
	\end{proof}

The choice for $F$ that we have in mind is the solution map of an SDE, which can be equivalently viewed as  an RDE driven by the enhanced (fractional) Brownian rough path. The reduction  of the H\"older exponent from $\alpha$ to $\beta$ is due to  a technical step in the proof, 
%cited compactness result, 
but this adjustment does not impact the validity of applying the theorem with this specific choice of $F$.

\begin{rem}[Universal approximations of Neural RDEs]
\label{rmk:universality_RDE}
It follows immediately from \autoref{thm:univ} and  the fact that the signature of a stochastic process satisfies a linear SDE (see \eqref{eq:sigRDE})
% and the fact that solutions of SDEs depend continuously   on the driving rough path, 
that Neural RDEs parametrised by feedforward neural networks with linear activation functions are dense in probability (in the sense of \eqref{eq:denseSig}) in the space of  all continuous functions on driving rough paths. 
 This indicates that smooth (and indeed linear) functions of solutions to
Neural RDEs can well-approximate the optimal controls of a non-Markovian control problem, 
as  the   optimal   strategies  typically  depend continuously on  the driven noise. In practice, one can expect superior performance (e.g.\ lower dimensions involved) when using non-linear activations, even though it is not needed for the theoretical result. This is because the non-linearity is already contained in operation of solving the SDE.
\end{rem}

\begin{rem}
    We would like to thank Terry Lyons for pointing out a shortcut to the proof of \autoref{thm:univ}, valid in the case of Brownian rough paths, which goes as follows: the Stratonovich rough path lift $X \mapsto \bfX$ is measurable, and thus by Lusin's theorem, for any $\delta > 0$ continuous on a compact set $K$ of probability $1-\delta$. We may then apply the classical Stone-Weierstrass theorem for signatures of rough paths with trace valued in $K$ to obtain the result.
\end{rem}

\section{Experiments}\label{sec:experiments}

We present a number of numerical experiments demonstrating the capabilities of our method to compute approximate solutions of non-Markovian stochastic control problems in continuous-time. We benchmark the performance of our approach against a selection of alternative RNN-based models parameterising the feed-back control \citet{han2021recurrent}. This choice of benchmarks is motivated mainly by the fact that this class of models, due to the connection between neural RDEs and RNNs, are the closest currently available alternatives to our method. The three alternative architectures we consider are: 1) RNN, 2) Long Short-Term Memory (LSTM), and 3) a Gated Recurrent Unit (GRU).  

One key feature of the proposed model that we wish to study empirically is the time-resolution invariance discussed at the end of Sec. \ref{sec:model-based}. Concretely, to test their robustness to changes in time-resolution, we train all models on a coarser time grid and then we evaluate them on a finer grid. Such a property is desirable both from the perspective of efficient training under computational budget constraints, and as an indication that the model is in fact learning a solution to the actual continuous-time problem. Lastly, a learned control that is heavily dependent on the grid that it was trained on may not be suitable for practical use; in such a case, this invariance property is critical.

Another performance criterion we will be using is the pathwise $L^2$ error between the state trajectories obtained using the NCDE control strategy and the ones obtained using the theoretical optimal control. The lower this error, the closer the trajectories sampled from the learnt state process to the trajectories of the theoretical state process.

All models are trained by sampling batches of trajectories from the state process under the parametric control, of the form given by \eqref{eqn:augmented_RDE}, and then performing direct backpropagation using an Adam optimiser \citep{kingma2014adam} to minimise the Monte-Carlo estimate of the value of the reward functional $J(\alpha^\theta)$ in \eqref{eqn:goal_functional}. For each experiment, the grid on which the system is simulated and the number of sample trajectories used to train and evaluate all models are kept the same. For a fair comparison, the hyperparameters of each model are adjusted such that the models all have an approximately equal number of trainable parameters. All experiments are implemented using version 1.11.0 of PyTorch and run on an NVIDIA Tesla K80 GPU. Additional experimental details can be found in the appendix.

% While the decrease in resolution and increase in the discretisation error of the simulation scheme may degrade performance slightly in and of itself, we expect a method with this resolution invariance property to remain relatively robust to such changes and not heavily dependent on being evaluated on the same grid it was trained on. Concretely, this would correspond to the final estimated value of the goal functional remaining relatively stable and not increasing dramatically as the training grid is made progressively coarser starting from the full evaluation grid. 

% The third and final experiment aims to provide a first empirical indication as to how the method may be applied to continuous-time non-Markovian reinforcement learning. {\color{red}[...]}

\subsection{Stochastic control problem with delay}\label{sec:LQ}
We consider first the example of a linear-quadratic problem with delay, also used by \citet{han2021recurrent}. Linear-quadratic control problems are widely used to address real-world challenges. An example from the mathematical finance community is the control of intraday fill ratios when volatility is stochastic; see  \citet{cartea2021shadow}.  In their paper, the control affects the state dynamics linearly and the performance criterion is composed of a square running penalty on the control and a square running penalty on one of the entries of the state process. A detailed discussion of this problem and how an explicit solution may be obtained is given by \citet{bauer2005stochastic}. 

With notation as before, the dynamics of the state process $X$ under a control $\alpha$ are given by
\begin{equation}\label{eq:experiments:dlqsde}
    \begin{aligned}
        \D X_t &= (A_1X_t + A_2Y_t + A_3 X_{t-\delta} + B \alpha_t)\D t + \sigma \D W_t, 
    \end{aligned}
\end{equation}
for $t\in [0,T]$ with $X_t = \phi$ for $t\in[-\delta, 0]$,  $\delta > 0$ a delay parameter, the distributed delay satisfying
\begin{equation*}
    Y_t := \int_{-\delta}^0 e^{\lambda \xi} h(X_{t+\xi})\D \xi, \quad t\in[0,T]
\end{equation*}
and the goal functional that we seek to minimise
% \begin{align*}
%     J(\alpha) &= \E\bigg[\int_0^T \Big((X_t + e^{\lambda \delta}A_3Y_t)^\top Q (X_t + e^{\lambda \delta}A_3Y_t) + \alpha^\top_t R\alpha_t \Big)\D t \\ &\qquad\qquad\qquad\qquad + (X_T + e^{\lambda \delta}A_3Y_T)^\top G (X_T + e^{\lambda \delta}A_3Y_T)\bigg].
% \end{align*}
\begin{align*}
    J(\alpha) &= \E\bigg[\int_0^T \Big(Z_t^\top Q Z_t + \alpha^\top_t R\alpha_t \Big)\D t + Z_T^\top G Z_T\bigg], \\
     Z_t :&= (X_t + e^{\lambda \delta}A_3Y_t),\ t\in[0,T].
\end{align*}
The parameters $A_1,A_2,A_3 \in \R^{d\times d}$, $B\in\R^{d\times d_\alpha}$, $\sigma\in \R^{d\times d_W}$, 
$Q,G\in \R^{d\times d}$,% positive semi-definite , 
$R\in\R^{d\times d}$,% positive definite, 
$\lambda,\delta$ and $T$ are all taken to be the same values as those used by \citet{han2021recurrent}. In particular, the problem is considered in 10 dimensions in state, noise and control, $Q,R,G$ are proportional to identity matrices, the elements of $A_1,A_3,B$ and $\sigma$ are selected randomly and $A_2$ is determined by a condition guaranteeing an explicit solution. We refer to \citet{han2021recurrent} for further details. The constant initial condition $\phi$ is taken to be zero. The explicit value function and optimal control are obtained in terms of the solution to an associated Riccati equation, which can be solved numerically.

\begin{table}[h]
    \centering
    \caption{\small \textbf{Linear-quadratic problem with delay}. Final estimate of the goal functional on the evaluation grid. Lower indicates smaller error. Analytical value: 2.231. Training resolution is given as a percentage of the evaluation resolution of 80 time steps.}
    \label{tab:experiments:delay}
    \vspace{2mm}
    \begin{tabular}{l c c c c}\toprule
         & \multicolumn{4}{c}{Training Resolution\vspace{1mm}}\\
        Model & $100\%$ & $50\%$ & $25\%$ & $12.5\%$ \\\midrule
        RNN & 2.493 & 5.162 & 8.870 & 7.600 \\
        LSTM & 2.357 & 7.323 & 5.888 & 6.863 \\
        GRU & 2.356 & 2.830 & 7.311 & 18.70 \\\midrule
        Neural RDE (ours) & 2.358 & 2.457 & 2.509 & 2.803 \\
        \bottomrule
    \end{tabular}
\end{table}

The results for this experiment are shown in table \ref{tab:experiments:delay}. We see that, trained at full resolution, the LSTM, GRU and Neural RDE models all perform approximately as well. However, when the training grid is made coarser, the Neural RDE model remains relatively stable with only slight increases in error, dramatically outperforming the benchmark models whose performance rapidly deteriorates. ~\\

\vspace{-0.5cm}

\subsection{Stochastic control problem driven by Fractional Brownian Motion}\label{sec:fBM}
Next, we demonstrate the application of the proposed method to a problem with non-Markovianity stemming  from correlated noise increments by considering a linear-quadratic problem driven by fractional Brownian motion. The dynamics for the state $X$ are given by
\begin{equation}
    \D X_t = (A X_t + C \alpha_t)\D t + \sigma\D W^H_t,\quad t\in [0,T], \qquad X_0 = 0
\end{equation}
where $A\in\R^{d\times d}$, $C\in\R^{d\times d_a}$, $\sigma\in\R^{d\times d_W}$ are parameters and $W^H$ is a $d_W$-dimensional fractional Brownian motion with components with Hurst parameters $H \in (0,1)$ (assumed the same across all $d_W$ channels). We choose the Hurst parameter $H=0.3$, so as to highlight the applicability of the method also in the case where solution paths are rougher than Brownian motion ($H=0.5$). 

The quadratic cost functional is as follows
\begin{equation}
    J(\alpha) = \frac{1}{2} \E \left[\int_0^T \big(X_s^\top Q X_s + \alpha_s^\top R \alpha_s\big) \D s + X_T^\top G X_T \right],
\end{equation}
where $Q,R,G$ are symmetric and positive definite. We consider the problem specifically in two dimensions in both state and control and take $T=1$, $\sigma= I$, 
\begin{align*}
    A &= \frac{1}{10}\begin{pmatrix}12 & 2\\ 2 & 12\end{pmatrix}, \quad C = \frac{1}{10}\begin{pmatrix}15 & -3\\ -3 & 15\end{pmatrix}, \\
    Q&=R=G = \frac{1}{10}I.
\end{align*}

 \vspace{-0.5cm}

%\begin{wraptable}[13]{l}{0.6\textwidth}
%\vspace{-2\baselineskip}
 \begin{table}[h]
    \centering
    \caption{\small \textbf{Linear-quadratic problem driven by fractional Brownian motion}. Final estimate of the goal functional on the evaluation grid. Lower indicates smaller error. Training resolution is given as a percentage of the evaluation resolution of 40 time steps.}
    \label{tab:experiments:fbm}
    \vspace{2mm}
    \begin{tabular}{l c c c c}\toprule
         & \multicolumn{4}{c}{Training Resolution\vspace{1mm}}\\
        Model & $100\%$ & $50\%$ & $25\%$ & $12.5\%$ \\\midrule
        RNN & 0.923 & 0.911 & 1.246 & 2.785 \\
        LSTM & 0.873 & 0.961 & 1.779 & 3.422 \\
        GRU & 0.891 & 0.925 & 1.236 & 2.791 \\\midrule
        Neural RDE (ours) & 0.896 & 0.902 & 0.927 & 1.104 \\
        \bottomrule
    \end{tabular}
 \end{table}
%\end{wraptable}

Table \ref{tab:experiments:fbm} shows the results for this experiment. We observe comparable performance between the LSTM, GRU and Neural RDE models at full training resolution, but with the Neural RDE significantly outperforming the other models when training resolution is decreased. At $12.5\%$ of evaluation resolution, the models are trained on simulations using just five time steps; nevertheless, the Neural RDE appears to produce reasonable results with an error compared to the full resolution case more than one order of magnitude smaller than for the other models.

\subsection{Portfolio optimisation problem with complete memory}\label{sec:portfolio}

We consider a portfolio optimisation problem with complete memory also studied in \citet{han2021recurrent}. A detailed analysis of this problem including derivations of explicit solutions under exponential, power and log utilities is given in \citet{pang2017stochastic}. Here, the state process $X_t$ represents the wealth of an investor and the $\alpha_t = (\alpha^1_t, \alpha_t^2)$ is a $2$-dimensional control process, where $\alpha^1_t$ is the amount of investment and $\alpha^2_t$ is the consumption of the underlying asset, i.e. the fraction of wealth consumed at time $t$. The dynamics are given, for $t\in [0,T]$, by
\begin{align}\label{eq:experiments:portfolio}
\D X_t &= (((\mu_1 - r)\alpha^1_t - \alpha^2_t+ r)X_t + \mu_2 Y_t)\D t \\
&\hspace{4cm} + \sigma \alpha^1_tX_t\D W_t,\nonumber\\
        Y_t :&= \int_{-\infty}^0 e^{\lambda \xi} X_{t + \xi}\D \xi,\nonumber
\end{align}
with $X_0 = \phi(0)$, $Y_0 = \int_{-\infty}^0 e^{\lambda \xi} \phi(t + \xi) \D \xi$, for some square integrable function $\phi$. The goal functional that we seem to maximise is as follows
\begin{align*}
    J(\alpha) &= \E\bigg[\int_0^T  e^{-\beta t} U_1(\alpha^2_tX_t) \D t  + e^{-\beta T U_2(X_T,Y_T)} \bigg], 
\end{align*}
where $U_1(x) = \log(x), U_2(x,y) = \frac{1}{\beta} \log(x + \eta y), \eta = \frac{1}{2}(\sqrt{(r + \lambda^2) + 4 \mu_2} - (r + \lambda))$.
As in the previous example all the parameters are taken to be the same as in \citet{han2021recurrent}.

%\begin{wraptable}[13]{l}{0.6\textwidth}
%\vspace{-2\baselineskip}
 \begin{table}[h]
    \centering
    \caption{\small \textbf{Portfolio optimisation with complete memory}. Relative difference between the estimated and the theoretical goal functionals as well as relative pathwise $L^2$ error between the true and estimated process trajectories. Lower indicates smaller error.}
    \label{tab:experiments:portfolio}
    \vspace{2mm}
    \begin{tabular}{l c c}\toprule
         & \multicolumn{2}{c}{Relative errors ($\times 10^{-3}$) \vspace{1mm}}\\
        Model & Goal functional & Pathwise $L^2$  \\\midrule
        RNN & 0.262 & 0.555  \\
        LSTM & 1.034 & 2.929  \\
        GRU & 0.541 & 1.116 \\\midrule
        Neural RDE (ours) & \textbf{0.238} & \textbf{0.043}  \\
        \bottomrule
    \end{tabular}
 \end{table}
%\end{wraptable}

The results for this experiment are shown in table \ref{tab:experiments:portfolio}, where we report the relative difference between the estimated and the theoretical goal functionals as well as relative pathwise $L^2$ error between the true and estimated process trajectories. We can see that the Neural RDE model slightly outperforms all alternative models on the relative difference of goal functionals and outperforms the second best model by one order of magnitude on the pathwise $L^2$ error. ~\\

\vspace{-0.5cm}

\section{Conclusion}

We proposed a framework for solving non-Markovian stochastic control problems continuous-time leveraging Neural RDEs. The main idea consists in parameterising the control process as the solution of a Neural RDE driven by the state process, so that the control-state joint dynamics are governed by an uncontrolled RDE with vector fields parameterised by neural networks. To deal with input paths of infinite 1-variation, we prove Theorem \ref{thm:univ}  which extends the universal approximation result in \citet{kidger2020neural} to Neural RDEs driven by random rough paths. We showcased the time-resolution-invariance of our approach on various non-Markovian problems, achieving better performance than traditional RNN-based approaches. 
%Finally, we discussed possible extensions of this framework to the setting of non-Markovian continuous-time reinforcement learning and provide promising empirical evidence in this direction.
A natural next step is to enhance the algorithm's efficiency by exploring the relationship between non-Markovian control problems and path-dependent  PDEs and BSDEs.

\bibliography{references}
\bibliographystyle{icml2023}

%%%%%%%%%%%%%%%%%%%%%%%%%%%%%%%%%%%%%%%%%%%%%%%%%%%%%%%%%%%%%%%%%%%%%%%%%%%%%%%
%%%%%%%%%%%%%%%%%%%%%%%%%%%%%%%%%%%%%%%%%%%%%%%%%%%%%%%%%%%%%%%%%%%%%%%%%%%%%%%
% APPENDIX
%%%%%%%%%%%%%%%%%%%%%%%%%%%%%%%%%%%%%%%%%%%%%%%%%%%%%%%%%%%%%%%%%%%%%%%%%%%%%%%
%%%%%%%%%%%%%%%%%%%%%%%%%%%%%%%%%%%%%%%%%%%%%%%%%%%%%%%%%%%%%%%%%%%%%%%%%%%%%%%
\newpage
\appendix
\onecolumn

\section{Appendix}

\subsection{Background on rough path theory}\label{sub:backRps}

The purpose of this appendix is to give an informal and very concise introduction to rough paths, their signatures, and their applications to machine learning.

For the abstract theory of rough paths we refer to \citet{FV10,FH20}. An $\alpha$-H\"older \emph{rough path} $\bfX$ consists of an $\alpha$-H\"older continuous path $X \colon [0,T] \to \bbR^d$ (the \emph{trace} of $\bfX$) together with a collection of higher-order functions defined on the simplex $\Delta[0,T] \coloneqq \{ (s,t) \in [0,T]^2 \mid 0 \leq s \leq t \leq T\}$ which represent, in a precise algebraic and analytic sense, iterated integrals of $X$ against itself. When $X$ is smooth or of bounded variation, such integrals can be defined canonically in the usual Stieltjes sense, and similarly when $X$ is $1/2 < \alpha$-H\"older continuous they can be defined via Young integration. However, when $\alpha \leq 1/2$ there is no canonical way of defining them, and if $X$ is a stochastic process, $\bfX$ is often defined through some notion of stochastic integration such as It\^o or Stratonovich. $\bfX$ takes values in $T^{\lfloor 1/\alpha \rfloor}(\bbR^d)$, where $T^N(\bbR^d) \coloneqq \bigoplus_{n = 0}^N (\bbR^d)^{\otimes n}$ denotes the tensor algebra over $\bbR^d$ truncated at level $N$ and $\lfloor \cdot \rfloor$ is the floor function: this  means, the rougher $X$ is, the more terms $\bfX$ must contain. Once such terms are defined, the \emph{signature} $S(\bfX)$ of $\bfX$ is canonically defined through well-known notions of path integration. $S(\bfX)$ is a map $\Delta[0,T] \to T(\!(\bbR^d)\!)$ (the algebra of formal series of tensors), and when $\alpha > 1/2$ it is canonically defined by Young integration as
\[
S(\bfX)_{st}^{(n)} \coloneqq \int_{s < u_1 < \ldots < u_n < t} \D X_{u_1} \otimes \cdots \otimes \D X_{u_n}
\]
where the superscript $(n)$ denotes projection onto $(\bbR^d)^{\otimes n}$. When $\alpha \leq 1/2$ the whole of $\bfX$, not just the trace $X$, is needed to define $S(\bfX)$, and $S(\bfX)^{(n)} = \bfX^{(n)}$ for $n \leq \lfloor 1/\alpha \rfloor$. We will denote $ C^\alpha([0,T],\bbR^d)$ the metrisable topological space of $\alpha$-H\"older rough paths taking values in $\bbR^d$ with time horizon $T$: this is what \citet{FV10} call $C^{\alpha\text{-H\"ol}}([0,T],G^{\lfloor 1/\alpha \rfloor}(\bbR^d))$; in \citet{FH20} (which only treats the case of $\alpha > 1/3$, nevertheless sufficient for Brownian motion, which is $\alpha$-H\"older regular for any $\alpha < 1/2$) this space is denoted $ C^\alpha_g([0,T],\bbR^d)$, the superscript $g$ standing for ``geometric''. Geometric rough paths are those which satisfy integration by parts relations, and are the only ones considered here; for example, It\^o and Stratonovich integration both define rough paths above Brownian motion, but only the latter is geometric. This is not an issue when considering It\^o SDEs, however, which can canonically be rewritten in Stratonovich form. The main example of rough path that we will consider is the Stratonovich Brownian rough path augmented with time: if $W$ is a $d$-dimensional Brownian motion, we take $\alpha$ to be any real number in $(1/3,1/2)$ and for $i,j = 1,\ldots d$ we let $\boldsymbol W^{ij}_{st} \coloneqq \int_{s < u < v < t} \circ\D W^i_u \circ\! \D W^j_v$, where $\circ \D W$ denotes Stratonovich integration. Time will take the zero-th coordinate, which means that when $i$ or $j$ above is $0$, the integral is defined through standard Young/Stieltjes integration.

The main purpose of rough path theory is to give meaning to \emph{rough differential equations} (RDEs) $\D Y = V(Y) \D \bfX$ which, in addition to having usual existence and uniqueness theorems, have the property that the solution map $\bfX \mapsto Y$ is continuous. This is not the case when considering SDEs: the map sending the Brownian sample path to the corresponding path of the solution, though defined on a set of full measure and measurable, is not continuous. An important RDE is the one satisfied by the signature itself on $T(\!(\bbR^d)\!)$: given a rough path $\bfX$ it holds that 
\begin{equation}\label{eq:sigRDE}
\D S(\bfX)_{0t} = S(\bfX)_{0t} \otimes \D \bfX_t
\end{equation}
The study of signatures is somewhat independent from that of rough paths, and is interesting even in the case of smooth or bounded variation paths (in which case $\bfX = X$). The main property of interest of the signature, established in \citet{HL10} (and extended to the full rough path case in \citet{BGLY16}), is that, for paths of bounded variation, the series of tensors $S(X)_{0T}$ determines the path $X$ up to \emph{treelike equivalence}. Roughly speaking, the latter means that if two paths $X$, $Y$ are such that $X \star \overleftarrow{Y}$ --- with $\star$ denoting path concatenation and $\overleftarrow{\phantom{Y}}$ denoting path inversion --- is a path that retraces itself and returns to the starting point, then the signature will not distinguish them: $S(X)_{0T} = S(Y)_{0T}$. We will write $X \sim Y$ for treelike equivalence (and similarly $\bfX \sim \boldsymbol Y$ in the generalised rough path sense of \citet{BGLY16}), and note that this includes (but is not limited to) the case in which $Y$ is a reparameterisation of $X$. ``Generic'' paths $\bbR^d$ valued paths can be expected not to be tree-like (i.e.\ not to retrace themselves) when $d > 1$; for example, in \citet{LQ13} it was shown that Brownian rough paths in dimension $2$ or greater a.s.\ do not contain tree-like pieces.

The result of \citet{HL10} is a powerful statement that makes it possible to understand a path $X \colon [0,T] \to \bbR^d$ in terms of the series of tensors $S(X)_{0T}$. What's more, the signature has the property of ``linearising'' all functions on paths: any non-linear function of $X$ can be approximately expressed as a linear functional on $S(X)$. A precise version of this statement in the random rough path case is proved in \ref{thm:univ} below. A fundamental ingredient for proving this type of result is the Stone-Weierstrass theorem: given a compact Hausdorff topological space $K$ and a subalgebra $A$ of $C(K,\bbR)$ which contains a non-zero constant function and separates points (this means that for any two distinct $x,y \in K$ there exists $a \in A$ s.t.\ $a(x) \neq a(y)$), it holds that $A$ is dense in $C(K,\bbR)$. The prototypical application of this theorem is the proof of density of polynomials in $C([a,b],\bbR)$. An important property that makes it possible to apply it to signatures is that linear functions on the signature, just like polynomials, form an algebra: if $\ell_1,\ell_2 \colon T(\!(\bbR^d)\!)^* = T(\bbR^d) \to \bbR$ are linear maps then
\[
\langle \ell_1,S(\bfX)_{0T} \rangle \langle \ell_1,S(\bfX)_{0T} \rangle = \langle \ell_1 \shuffle \ell_2 ,S(\bfX)_{0T} \rangle
\]
where $\shuffle$ is the combinatorial operation of shuffling. This relation can be understood as a generalised integration by parts relation, as can be seen by taking $\ell_1$ and $\ell_2$ to be evaluations against elementary tensors: in this case (and $X$ of bounded variation) the above identity reads
\begin{align*}
	&\mathrel{\phantom{=}}\bigg( \int_{s < u_1 < \ldots < u_n < t} \D X^{i_1}_{u_1} \cdots \D X^{i_m}_{u_m} \bigg) \bigg( \int_{s < v_1 < \ldots < v_n < t} \D X^{j_1}_{v_1} \cdots \D X^{j_n}_{v_n} \bigg) \\
	&= \sum_{\boldsymbol k \in \text{Sh}(\boldsymbol i, \boldsymbol j)} \int_{s < r_1 < \ldots < r_{n+m} < t} \D X^{k_1}_{r_1} \cdots \D X^{k_{m+n}}_{r_{m+n}}
\end{align*}
where we are summing over all multiindices $\boldsymbol k$ obtained by shuffling the multiindices $(i_1,\ldots,i_m)$ and $(j_1,\ldots,j_m)$. For these reasons, signatures have been extensively used for in the context of machine learning for time series, see e.g.\ \citet{CK16, fermThesis}. 

\subsection{Additional experimental details}

In this final section of the appendix we present additional experimental details.

\paragraph{Stochastic control problem with delay (sec. \ref{sec:LQ})} Trajectories of  \eqref{eq:experiments:dlqsde} are simulated using an Euler-Maruyama-type scheme on a uniform grid. All models are trained over 300 batches of 256 sample trajectories simulated on grid with 80, 40, 20 and 10 time steps. The final evaluation estimates of the goal functional are computed using 4096 sample trajectories simulated on a grid with 80 time steps. The dimension of the hidden states in the baseline models are: 400 for the RNN, 200 for the LSTM and 230 for the GRU. The latent dimension of the Neural RDE model is 200 and the vector field and initial lift are parameterised by fully connected feed-forward neural networks with two hidden layers of width $64$. We take activations given by elementwise application of the SiLU function $x \mapsto \tfrac{x}{1+e^{-x}}$ and apply a final $\tanh$ non-linearity to the outputs to prevent unreasonably large values and initial losses.

\paragraph{Stochastic control problem driven by fractional Brownian motion (sec. \ref{sec:fBM})} Trajectories are simulated using an Euler-Maruyama scheme with increments of fractional Brownian motion sampled using the Python package fbm \citep{fBM}. We use uniform grids with 40 time steps for evaluation and 40, 20, 10 and 5 steps for training. All models are trained over 300 batches of 256 sample trajectories. The final evaluation estimates of the value of the goal functional are computed using 4096 sample trajectories. The dimension of the hidden states in the baseline models are: 250 for the RNN, 130 for the LSTM and 150 for the GRU. The latent dimension of the Neural RDE model is 200 and the vector field and initial lift are parameterised by feed-forward neural networks with two hidden layers of width $64$ respectively.

\paragraph{Portfolio optimisation problem with complete memory (sec. \ref{sec:portfolio})} We simulate trajectories of  \eqref{eq:experiments:portfolio} using an Euler-Maruyama-type scheme on a uniform grid of $200$ time steps. All models are trained over $500$ batches of $256$ sample trajectories (similarly for evaluation). The dimension of the hidden states in the baseline models are: $600$ for the RNN, $300$ for the LSTM and $300$ for the GRU. The latent dimension of the Neural RDE model is $350$ and the vector field and initial lift are parameterised by  feed-forward neural networks with two hidden layers of width $128$.

\end{document}

%% file: maths_commands.tex
\usepackage{amsmath,amsfonts,bm}

\newcommand{\D}{\mathrm{d}}
\newcommand{\EE}{\mathbb{E}}
\def\eqref#1{equation~(\ref{#1})}
\def\1{\bm{1}}

\DeclareMathAlphabet{\mathsfit}{\encodingdefault}{\sfdefault}{m}{sl}
\SetMathAlphabet{\mathsfit}{bold}{\encodingdefault}{\sfdefault}{bx}{n}

\def\sR{{\mathbb{R}}}

\newcommand{\E}{\mathbb{E}}

\newcommand{\R}{\mathbb{R}}